\documentclass[letterpaper]{article} 
\usepackage{aaai25}  
\usepackage{times}  
\usepackage{helvet}  
\usepackage{courier}  
\usepackage[hyphens]{url}  
\usepackage{graphicx} 
\usepackage{natbib}  
\usepackage{caption} 
\usepackage{algorithm}
\usepackage{algorithmic}

\usepackage{newfloat}
\usepackage{listings}
\DeclareCaptionStyle{ruled}{labelfont=normalfont,labelsep=colon,strut=off} 
\floatstyle{ruled}
\newfloat{listing}{tb}{lst}{}
\floatname{listing}{Listing}
\title{Neural Networks Perform Sufficient Dimension Reduction}
\author{
    Shuntuo Xu, Zhou Yu\footnote{Corresponding author.}
}
\affiliations{
    Key Laboratory of Advanced Theory and Application in Statistics and Data Science -- MOE \\
    School of Statistics, East China Normal University, Shanghai, China \\
    oaksword@163.com, zyu@stat.ecnu.edu.cn\\

}

\usepackage{amsmath}
\usepackage{amssymb}
\usepackage{booktabs}
\usepackage{amsthm}

\newcommand{\T}{\top}
\newcommand{\E}{\mathbb{E}}
\newcommand{\var}{\mathrm{Var}}

\newcommand{\indep}{\perp\!\!\!\perp}
\newcommand{\real}{\mathbb{R}}
\newcommand{\argmin}[1]{\mathop{\mathrm{argmin}}\limits_{#1}}

\newtheorem{theorem}{Theorem}
\newtheorem{assumption}{Assumption}
\newtheorem{lemma}{Lemma}

\begin{document}

\maketitle

\begin{abstract}
This paper investigates the connection between neural networks and sufficient dimension reduction (SDR), demonstrating that neural networks inherently perform SDR in regression tasks under appropriate rank regularizations. Specifically, the weights in the first layer span the central mean subspace. We establish the statistical consistency of the neural network-based estimator for the central mean subspace, underscoring the suitability of neural networks in addressing SDR-related challenges. Numerical experiments further validate our theoretical findings, and highlight the underlying capability of neural networks to facilitate SDR compared to the existing methods. Additionally, we discuss an extension to unravel the central subspace, broadening the scope of our investigation.
\end{abstract}

\section{Introduction}

Neural networks have achieved significant success in a tremendous variety of applications \citep{lee2017large, silver2018general, jumper2021highly, brandes2022proteinbert, thirunavukarasu2023large}. As the most essential foundation, a feedforward neural network is typically constructed by a series of linear transformations and nonlinear activations. To be specific, a function $f$ implemented by a feedforward neural network with $L$ layers can be represented as \begin{equation} 
f(x)=\phi_L\circ\sigma_{L-1}\circ\phi_{L-1}\circ\cdots \circ\sigma_0\circ\phi_0(x). 
\label{Eq: nn_form} 
\end{equation} 
Here, $\circ$ is the functional composition operator, $\phi_i(z)=W_i^{\T}z+b_i$ denotes the linear transformation and $\sigma_i$ represents the elementwise nonlinear activation function. Despite a clear architecture, the formula \eqref{Eq: nn_form} provides limited insight into how the information within the input data is processed by the neural networks. To comprehensively understand how neural networks retrieve task-relevant information, there have been extensive efforts towards unraveling their interpretability \citep{doshivelez2017rigorousscienceinterpretablemachine, guidotti2018survey, zhang2021survey}. For instance, post-hoc interpretability \citep{lipton2018mythos} focused on the predictions generated by neural networks, while disregarding the detailed mechanism and feature importance. \citet{ghorbani2019interpretation} highlighted the fragility of this type of interpretation, as indistinguishable perturbations could result in completely different interpretations.

For the sake of striking a balance between the complexity of representation and the power of prediction, \citet{tishby2015deep} and \citet{saxe2019information} pioneered the interpretability of deep neural networks via the information bottleneck theory. \citet{ghosh2022sufficient} further linked the information bottleneck to sufficient dimension reduction (SDR) \citep{li1991sliced, cook1996graphics}, which is a rapidly developing research field in regression diagnostics, data visualization, pattern recognition and etc.

In this paper, we provide a theoretical understanding of neural networks for representation learning from the SDR perspective. Let $x\in\real^p$ represent the covariates and $y\in\real$ represent the response. Consider the following regression model 
\begin{equation} 
y=f_0(B_0^{\T}x)+\epsilon, 
\label{Eq: main_model} 
\end{equation} 
where $B_0\in\real^{p\times d}$ is a nonrandom matrix with $d\le p$, $f_0:\real^d\to\real$ is an unknown function, and $\epsilon$ is the noise such that $\E(\epsilon| x)=0$ and $\var(\epsilon| x)=\nu^2$ for some positive constant $\nu$. {{Intuitively, the semiparametric and potentially multi-index model \eqref{Eq: main_model} asserts that the core information for regression is encoded in the low-dimensional linear representation $B_0^{\T}x$. In the literature of SDR, model \eqref{Eq: main_model} has been extensively studied. Based on model \eqref{Eq: main_model} , \citet{cook2002dimension} proposed the objective of sufficient mean dimension reduction as 
\begin{equation} 
y\indep \E(y| x) | B^{\T}x
\label{Eq: CMS} 
\end{equation} 
for some matrix $B$, where $\indep$ means statistical independence.  Denote the column space spanned by $B$ as $\Pi_{B}$, which is commonly referred to as the mean dimension-reduction subspace. It is evident that the matrix $B$ satisfying \eqref{Eq: CMS} is far away from uniqueness. Hence, we focus on the intersection of all possible mean dimension-reduction subspace, which is itself a mean dimension-reduction subspace under mild conditions (e.g., when the domain of $x$ is open and convex; see \citet{cook2002dimension}), and term it the central mean subspace. Agreeing to condition \eqref{Eq: CMS}, $B_0$ defined in model \eqref{Eq: main_model} yields the central mean subspace, denoted as $\Pi_{B_0}$, under certain assumption. Statistical estimation and inference about the central mean subspace is the primary goal of SDR. Popular statistical methods for recovering central mean subspace include ordinary least squares \citep{li1989regression}, principal Hessian directions \citep{li1992principal}, minimum average variance estimation \citep{xia2002adaptive}, semiparametric approach \citep{ma2013efficient}, generalized kernel-based dimension reduction \citep{fukumizu2014gradient}, and many others.}} Although numerous studies have demonstrated the ability of neural networks to approximate complex functions \citep{hornik1989multilayer, barron1993universal, yarotsky2017error, shen2021neural} and to adapt low-dimensional structures \citep{bach2017breaking, bauer2019on, schmidt2020nonparametric, abbe2022merged,jiao2023deep, troiani2024fundamental}, it is of subsequent interest to investigate whether neural networks are capable of correctly identifying the intrinsic structure encapsulated in $B_0$, thereby deepening the interpretation of neural networks.

Our study was inspired by an observation that the weight matrix in the first layer, i.e., $W_1$, could accurately detect the presence of $B_0$ in a toy data set, with the rank regularization. Specifically, consider the toy model $y=(B_0^{\T}x)^3+\epsilon$ where $x\sim \mathrm{Uniform}([0, 1]^{10})$, $B_0=(1, -2, 0, \ldots, 0)^{\T}$ and $\epsilon\sim \mathrm{Normal}(0, 0.1^2)$. We trained neural networks using the least squares loss with $W_1=W_{11}W_{12}$ where $W_{11}\in\real^{10\times q}$ and $W_{12}\in\real^{q\times 64}$ for $q=1, \ldots , 10$. Evidently, the rank of $W_1$ did not exceed $q$. It was then observed that for each $q$, (i) $B_0$ was closely contained within the column space of $W_{11}$, as the absolute cosine similarity between $B_0$ and its projection on $\Pi_{W_{11}}$ was close to 1, and (ii) the leading eigenvector of $W_1W_1^{\T}$ closely aligned with $B_0$ (see Figure \ref{Fig: motivation}). This observation indicates that the first layer of the neural network may potentially discover the underlying low-dimensional intrinsic structure. 

\begin{figure}[b]
    \centering
    \includegraphics[width=5.5cm]{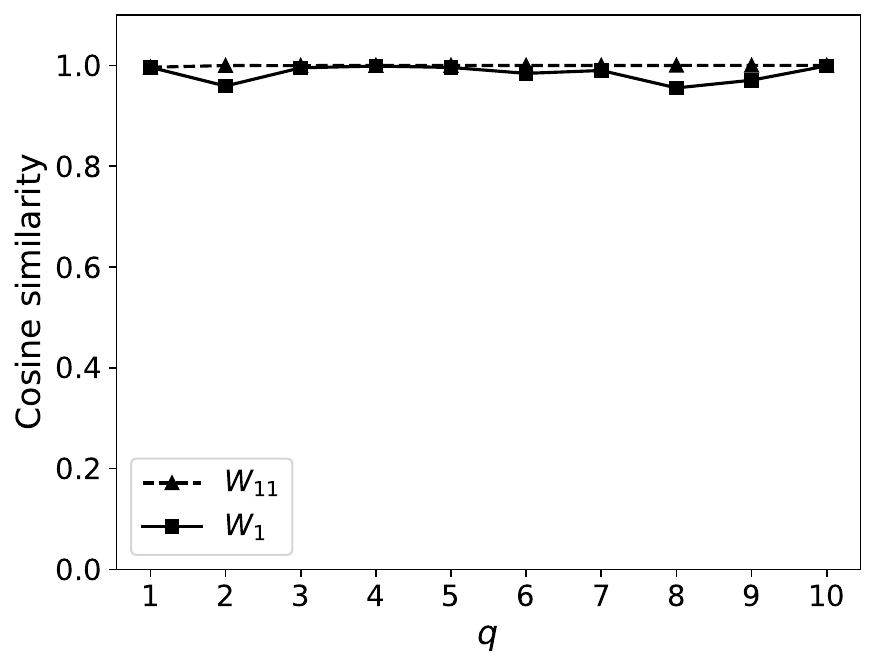}
    \caption{Absolute cosine similarity between (i) $B_0$ and its projection on $\Pi_{W_{11}}$ (the dot line with triangle marks), (ii) $B_0$ and the leading eigenvector of $W_1W_1^{\T}$ (the solid line with square marks).}
    \label{Fig: motivation}
\end{figure}

It is important to note that the application of neural networks for estimating the central mean subspace based on model \eqref{Eq: main_model} was previously explored by \citet{kapla2022fusing} through a two-stage method. The first stage focused on obtaining a preliminary estimation of $\Pi_{B_0}$, which subsequently served as an initial point for the joint estimation of $\Pi_{B_0}$ and $f_0$ in the second stage. Given our toy example, however, it is prudent to critically evaluate the necessity of the first stage. Furthermore, their work lacks a comprehensive theoretical guarantee. Additionally, another related work conducted by \citet{liang2022nonlinear} concentrated on seeking nonlinear sufficient representations. In contrast, we focus more on revealing the fundamental nature of neural networks. 

We in this paper show that, with suitable rank regularization, the first layer of a feedforward neural network conducts SDR in a regression task, wherein $d(W_{11}, B_0)\to 0$ in probability for certain distance metric $d(\cdot, \cdot)$. Furthermore, numerical experiments provide empirical evidences supporting this result, while demonstrating the efficiency of neural networks in addressing the issue of SDR.

Throughout this paper, we use $\|v\|_2$ to represent the Euclidean norm of a vector $v$. For a matrix $A$, $\|A\|_F$ is the Frobenius norm of $A$, $\pi_A=A(A^{\T}A)^-A^{\T}$ denotes the projection matrix of $A$ where $A^-$ is the generalized inverse of $A$, and $\Pi_A$ stands for the linear space spanned by the columns of $A$. For a measurable function $f: \mathcal{X}\to \real$, $\|f\|_{L^2(\mu)}$ represents the $L^2$ norm of $f$ with respect of a given probability measure $\mu$, and $\|f\|_{L^{\infty}(\mathcal{X})}=\sup_{x\in\mathcal{X}}|f(x)|$ represents the supreme norm of $f$ over a set $\mathcal{X}$. $\mathcal{B}(\mathcal{X})$ is the unit ball induced by $\mathcal{X}$, such that $\mathcal{B}(\mathcal{X})\subset\mathcal{X}$ and $\|v\|_2\le 1$ for any $v\in\mathcal{X}$.

\section{Theoretical Justificatitons}

\subsection{Population-Level Unbiasedness}

Suppose the true intrinsic dimension $d$ defined in model \eqref{Eq: main_model} is known and the covarites $x\in\mathcal{B}([0, 1]^p)$. As $B_0$ is not identifiable in (\ref{Eq: main_model}) and (\ref{Eq: CMS}) , it is assumed without loss of generality that $B_0^{\T}B_0=I_d$ where $I_d$ is the identity matrix with $d$ rows. By defining $\Psi_d=\{B\in\real^{p\times d}: B^{\T}B=I_d\}$, we have $B_0\in\Psi_d$. In this paper, we consider the following neural network function class
$$
\begin{aligned}
&\mathcal{F}_{\mathcal{L}, \mathcal{M}, \mathcal{S}, \mathcal{R}}=\Big\{ \\
& f(x)=\phi_L\circ\sigma_{L-1}\circ\phi_{L-1}\circ\cdots \circ\sigma_0\circ\phi_0(B^{\T}x): \\
& B\in\Psi_d, \phi_i(z)=W_i^{\T}z+b_i, W_i\in\real^{d_i\times d_{i+1}}, i=0, \ldots, L, \\
&\text{with } L=\mathcal{L}, \max_{i=1, \dots, L}d_i=\mathcal{M}, \sum_{i=0}^L(d_i+1)d_{i+1}=\mathcal{S}, \\
& \|f\|_{L^{\infty}(\mathcal{B}([0, 1]^p))}\le \mathcal{R}
\Big\}.
\end{aligned}
$$
The activation functions $\sigma_i (i=0, \ldots, L-1)$ utilized are the rectified linear units, i.e., $\sigma_i(x)=\max(x, 0)$. We emphasize that $\mathcal{F}_{\mathcal{L}, \mathcal{M}, \mathcal{S}, \mathcal{R}}$ incorporates rank regularization of $d$ in the first layer. For arbitrary $f\in\mathcal{F}_{\mathcal{L}, \mathcal{M}, \mathcal{S}, \mathcal{R}}$, we use $\mathcal{T}(f)$ to represent the component $B\in\Psi_d$ in the first layer of $f$.

For a regression task, the theoretical study is highly related to the smoothness of  underlying conditional mean function \citep{yang1999information}. Here, we introduce the following assumptions of model \eqref{Eq: main_model}.

\begin{assumption}[Smoothness]
$f_0$ is a H\"older continuous function of order $\alpha\in (0, 1]$ with constant $\lambda$, i.e., $|f_0(x)-f_0(z)|\le \lambda\|x-z\|_2^{\alpha}$ for any $x, z\in [0, 1]^d$. Additionally, $\|f_0\|_{L^{\infty}([0, 1]^d)}\le \mathcal{R}_0$ for some constant $\mathcal{R}_0\ge 1$.
\label{Asm: f0_smoothness}
\end{assumption}

\begin{assumption}[Sharpness]
For any scalar $\delta>0$ and $B\in\Psi_d$, $\|\pi_B-\pi_{B_0}\|_F>\delta$ implies $\E\{\var[f_0(B_0^{\T} x)| B^{\T} x]\}>M(\delta)$ for some $M(\delta)>0$.
\label{Asm: f0_sharpness}
\end{assumption}

\begin{assumption}
$y$ is sub-exponentially distributed such that there exists $\tau>0$ satisfying $\E[\exp(\tau|y|)]<\infty$.
\label{Asm: y_subexp}
\end{assumption}

Assumption \ref{Asm: f0_smoothness} is a technical condition to study the approximation capability of neural networks \citep{shen2019deep}. Alternatively, other functional spaces, such as the Sobolev space, can also be employed for this purpose \citep{abdeljawad2022approximations, shen2022approximation}. Furthermore, Assumption \ref{Asm: f0_sharpness} establishes a restriction on the sharpness of $f_0$. Consider the case that $f_0$ is solely a constant function as zero. Then, a trivial neural network by setting all the parameters except $B$ to zero perfectly fits $f_0$, regardless of the value of $B$. With Assumption \ref{Asm: f0_sharpness}, it becomes difficult to accurately capture the overall behavior of $f_0(B_0^{\T} x)$ using a biased $B$. In other words, $B_0^{\T}x$ is sufficient and necessary for recovering $\E(y| x)$, i.e., $\Pi_{B_0}$ is the central mean subspace. Similar condition was also adopted in Theorem 4.2 of \citet{li2009dimension} to distinguish sufficient directions $B_0$ from other $B$. Assumption \ref{Asm: y_subexp} is a commonly used condition for applying empirical process tools and concentration inequalities \citep{van2000asymptotic, zhu2022robust}. 

\begin{table*}[t]
\centering
\begin{tabular}{lccrrrrrr}
\toprule
 &  &  & NN & MAVE & GKDR & SIR & SAVE & PHD \\ \midrule
 & $(n, p)=(100, 10)$ & mean & 0.135 & 0.160 & 0.388 & 0.897 & 0.315 & 0.623 \\
 &  & std & 0.064 & 0.056 & 0.126 & 0.186 & 0.082 & 0.150 \\
Setting 1 & $(n, p)=(200, 30)$ & mean & 0.276 & 0.214 & 1.012 & 0.946 & 0.298 & 0.807 \\
 &  & std & 0.274 & 0.051 & 0.299 & 0.113 & 0.049 & 0.107 \\
 & $(n, p)=(300, 30)$ & mean & 0.120 & 0.130 & 0.542 & 0.900 & 0.337 & 0.709 \\ 
 &  & std & 0.038 & 0.029 & 0.166 & 0.136 & 0.063 & 0.108 \\ \midrule
 & $\sigma=0.1$ & mean & 0.296 & 0.730 & 1.130 & 0.665 & 0.319 & 1.550 \\
 &  & std & 0.126 & 0.312 & 0.271 & 0.166 & 0.076 & 0.126 \\
Setting 2 & $\sigma=0.2$ & mean & 0.628 & 0.899 & 1.155 & 0.705 & 0.338 & 1.526 \\
 &  & std & 0.269 & 0.329 & 0.237 & 0.149 & 0.074 & 0.138 \\
 & $\sigma=0.5$ & mean & 1.197 & 1.187 & 1.278 & 0.830 & 0.367 & 1.567 \\
 &  & std & 0.213 & 0.204 & 0.200 & 0.169 & 0.072 & 0.131 \\ \midrule
 & $n=200$ & mean & 0.639 & 1.246 & 1.759 & 1.669 & 1.728 & 1.740 \\ 
 &  & std & 0.418 & 0.289 & 0.149 & 0.235 & 0.136 & 0.221 \\
Setting 3 & $n=500$ & mean & 0.248 & 1.076 & 1.752 & 1.650 & 1.683 & 1.721 \\
 &  & std & 0.242 & 0.331 & 0.125 & 0.271 & 0.201 & 0.228 \\
 & $n=1000$ & mean & 0.075 & 0.924 & 0.554 & 1.652 & 1.678 & 1.737 \\
 &  & std & 0.081 & 0.382 & 0.371 & 0.259 & 0.245 & 0.191 \\ \midrule
 & $n=1000, d=4$ & mean & 0.127 & 0.293 & 0.368 & 1.363 & 0.429 & 0.960 \\
 &  & std & 0.161 & 0.050 & 0.079 & 0.167 & 0.079 & 0.178 \\
Setting 4 & $n=1500, d=6$ & mean & 0.144 & 0.415 & 0.386 & 1.530 & 0.467 & 0.902 \\
 &  & std & 0.067 & 0.076 & 0.077 & 0.153 & 0.082 & 0.206 \\
 & $n=2000, d=8$ & mean & 0.140 & 0.360 & 0.344 & 1.410 & 0.364 & 0.735 \\
 &  & std & 0.055 & 0.084 & 0.010 & 0.163 & 0.065 & 0.175 \\
\bottomrule
\end{tabular}
\caption{The results of average and standard deviation of $\|\pi_{\hat{B}}-\pi_{B_0}\|_F$ on 100 replicates across different methods. NN represents the neural network-based method.}
\label{Tab: simulation}
\end{table*}

\begin{theorem}
    Suppose that Assumptions \ref{Asm: f0_smoothness} and \ref{Asm: f0_sharpness} hold. Let 
    $
    f^*=\mathrm{argmin}_{f\in\mathcal{F}_{\mathcal{L}, \mathcal{M}, \mathcal{S}, \mathcal{R}}}\E[y-f(x)]^2$, then
    $$\Pi_{\mathcal{T}(f^*)}=\Pi_{B_0},$$
     provided that $\mathcal{R}$ is sufficiently large, and $\mathcal{L}$ and $\mathcal{M}$ tend to infinity. 
     \label{Thm: population_minimum}
\end{theorem}

Theorem \ref{Thm: population_minimum} builds a bridge connecting neural networks and SDR.  It demonstrates that neural networks indeed achieve representation learning as the first layer of the optimal neural network perfectly reaches the target of SDR at population level. Theorem \ref{Thm: population_minimum} also inspires us to perform SDR based on neural networks with a minor adjustment for the first layer. The detailed proof of Theorem \ref{Thm: population_minimum} can be found in Section Proofs.

\subsection{Sample Estimation Consistency}
We now investigate the theoretical property of the neural network-based sample-level estimator for SDR. Given sample observations $\mathcal{D}_n=\{(X_1, Y_1), \ldots, (X_n, Y_n)\}$, where $(X_i, Y_i)$ is an independent copy of $(x, y)$ for $i=1, \ldots, n$, the commonly used least squares loss is adopted, i.e.,
$$
L_n(f)=\frac 1n\sum_{i=1}^n[Y_i-f(X_i)]^2.
$$ 
Denote the optimal neural network estimator at the sample level as
$$
\hat{f}_n=\argmin{f\in \mathcal{F}_{\mathcal{L}, \mathcal{M}, \mathcal{S}, \mathcal{R}}}L_n(f).
$$ 
$\mathcal{T}(\hat f_n)$ is then the sample estimator approximately spanning the central mean subspace.   

To examine the closeness between $\Pi_{\mathcal{T}(\hat{f}_n)}$ and $\Pi_{B_0}$, we define the distance metric $d(\cdot, \cdot)$ as
$$
d(B, B_0)=\min_{Q\in \mathcal{Q}}\|B_0-B Q\|_2,
$$
where $B\in\Psi_d$ and $\mathcal{Q}$ is the collection of all orthogonal matrices in $\real^{d\times d}$. We see that $d(B, B_0)=0$ if and only if $\Pi_{B}=\Pi_{B_0}$. And we make the following assumption essentially another view of Assumption \ref{Asm: f0_sharpness}. 

\begin{assumption}
For any positive scalar $\delta$, $d(B, B_0)>\delta$ implies $\E\{\var[f_0(B_0^{\T} x)| B^{\T} x]\}>M_1(\delta)$ for some $M_1(\delta)>0$. 
\label{Asm: metric_sharpness}
\end{assumption}

\begin{theorem}
Suppose the Assumptions \ref{Asm: f0_smoothness}, \ref{Asm: y_subexp} and \ref{Asm: metric_sharpness} hold. Then we have
$$
d(\mathcal{T}(\hat{f}_n), B_0)\stackrel{P}\longrightarrow 0,
$$
when the depth $\mathcal{L}=O(n^{d/(2d+8\alpha)})$, the width $\mathcal{M}=O(1)$ and $\mathcal{R}\ge \mathcal{R}_0$.
\label{Thm: convergence}
\end{theorem}

Theorem \ref{Thm: convergence} confirms that $\Pi_{\mathcal{T}(\hat{f}_n)}$ converges to $\Pi_{B_0}$ in probability. As a result, under the least squares loss, the neural networks, with appropriate rank regularization, provide a consistent estimator of the central mean subspace. Therefore, it is promising to adopt the neural network function class $\mathcal{F}_{\mathcal{L}, \mathcal{M}, \mathcal{S}, \mathcal{R}}$ to address sufficient mean dimension reduction problems. More importantly, this approach offers several advantages over existing methods. Unlike SIR \citep{li1991sliced}, SAVE \citep{cook1991sliced} and PHD \citep{li1992principal}, our method does not impose stringent probabilistic distributional conditions on $x$, such as linearity, constant variance and normality assumptions. Compared to MAVE \cite{xia2002adaptive}, our method adopts the powerful neural networks to estimate the link function, thereby working similar or better than classical nonparametric tools based on first-order Taylor expansion. We present the proof of Theorem \ref{Thm: convergence} in Section Proofs.

The results illustrated in Theorem \ref{Thm: population_minimum} and \ref{Thm: convergence} are contingent on the availability of true intrinsic dimension $d$. In the case where $d$ is unknown, a natural modification for $\mathcal{F}_{\mathcal{L}, \mathcal{M}, \mathcal{S}, \mathcal{R}}$ is to set $B\in\real^{p\times p}$ without rank regularization. Under Assumptions \ref{Asm: f0_smoothness} and \ref{Asm: f0_sharpness}, in this scenario, the optimal $f^*$ at the population level still ensures that $\Pi_{\mathcal{T}(f^*)}$ encompasses $\Pi_{B_0}$, where the sharpness of $f_0$ plays a crucial role; see the Supplementary Material for more details. 

\section{Simulation Study}

From the perspective of numerical experiments, we utilized simulated data sets to demonstrate that (i) the column space of $\mathcal{T}(\hat{f}_n)$ approached the central mean subspace $\Pi_{B_0}$ as sample size $n$ increased, (ii) the performance of neural network-based method in conducting SDR was comparable to classical methods. In some cases, the neural network-based method outperformed classical methods, particularly when the latent intrinsic dimension $d$ was large. Five commonly used methods, sliced inverse regression (SIR), sliced average variance estimation (SAVE), principal Hessian directions (PHD), minimum average variance estimation (MAVE) and generalized kernel dimension reduction (GKDR), were included for comparisons.

The neural network-based method was implemented in Python using PyTorch. Specifically, we utilized a linear layer without bias term to represent the matrix $B$ (recall that we suppose $d$ is known), which was further appended by a fully-connected feedforward neural network with number of neurons being $h-h/2-1$. Here, we set $h=64$ when the sample size was less than $1000$, and set $h=128$ when sample size was between 1000 and 2000. Overall, the neural network architecture was $p-d-h-h/2-1$. Code is available at \url{https://github.com/oaksword/DRNN}.

We considered the following four scenarios, with two of them attaining small $d=1, 2$ and the rest of $d\ge 3$.

Setting 1: $y=x_1^4+\epsilon$ where $x\sim \mathrm{Normal}(0, I_p)$ and $\epsilon\sim \mathrm{t}_5$.

Setting 2: $y=\log(x_1+x_1x_2)+\epsilon$ where $x\sim \mathrm{Uniform}([0, 1]^{10})$ and $\epsilon\sim \mathrm{Normal}(0, \sigma^2)$.

Setting 3: $y=(1+\beta_1^{\T} x)^2\exp(\beta_2^{\T} x)+5\cdot 1(\beta_3^{\T} x>0)+\epsilon$ where $x\sim \mathrm{Uniform}([-1,1]^6)$, $\epsilon\sim \chi^2_2-2$, $\beta_1=(-2, -1, 0, 1, 2, 3)^{\top}$, $\beta_2=1_6$ and $\beta_3=(1, -1, 1, -1, 1, -1)^{\top}$.

Setting 4:
$$
y=\sum_{k=1}^d\frac{e^{x_i}}{2+\sin(\pi x_i)}+\epsilon,
$$ 
where $\epsilon\sim \mathrm{Normal}(0, 0.1^2)$, 
$$
x\sim \mathrm{Normal}\left(1_{10}, I_{10}-1_{10}1_{10}^{\T}/20\right).
$$

In setting 1, we tested the combinations of $(n, p)=(100, 10), (200, 30)$ and $(300, 30)$, where $n$ was the sample size. In setting 2, we fixed $n=200, p=10$ and set $\sigma=0.1, 0.2, 0.5$. In setting 3, we fixed $p=6$ and tested $n=200, 500, 1000$. In setting 4, we fixed $p=10$ and tested $(n, d)=(1000, 4), (1500, 6)$ and $(2000, 8)$. Settings 1, 2, 4 were equipped with continuous regression functions, and setting 3 involved discontinuity. We set the number of iterations of our method to 1000 for settings 1 and 2, and increased the iterations to 2000 and 4000 for settings 3 and 4, due to their high complexity.

To evaluate the performance of each method, we calculated the distance metric $\|\pi_{\hat{B}}-\pi_{B_0}\|_F$ where $\hat{B}$ represented the estimate of $B_0$. Particularly, for our neural network-based method, $\hat{B}=\mathcal{T}(\hat{f}_n)$. Smaller value of this metric indicated better performance. We ran 100 replicates for each setting. The results are displayed in Table \ref{Tab: simulation} and Figure \ref{Fig: simulation}.

In simple cases (settings 1 and 2), the neural network-based method showed similar performance to classical methods, with sliced average variance estimation being the most effective method in setting 2. However, complex settings demonstrated the significant superiority of neural network-based method compared to other methods. Specifically, in setting 3, as the sample size $n$ increased, the metric $\|\pi_{\hat{B}}-\pi_{B_0}\|_F$ decreased rapidly, indicating the convergence of $\Pi_{\hat{B}}$ to $\Pi_{B_0}$. According to the results in setting 4, the neural network-based method was capable of handling higher-dimensional scenarios. In summary, numerical studies advocated neural networks as a powerful tool for SDR.

\begin{figure*}[ht]
\centering 
\includegraphics[width=14cm]{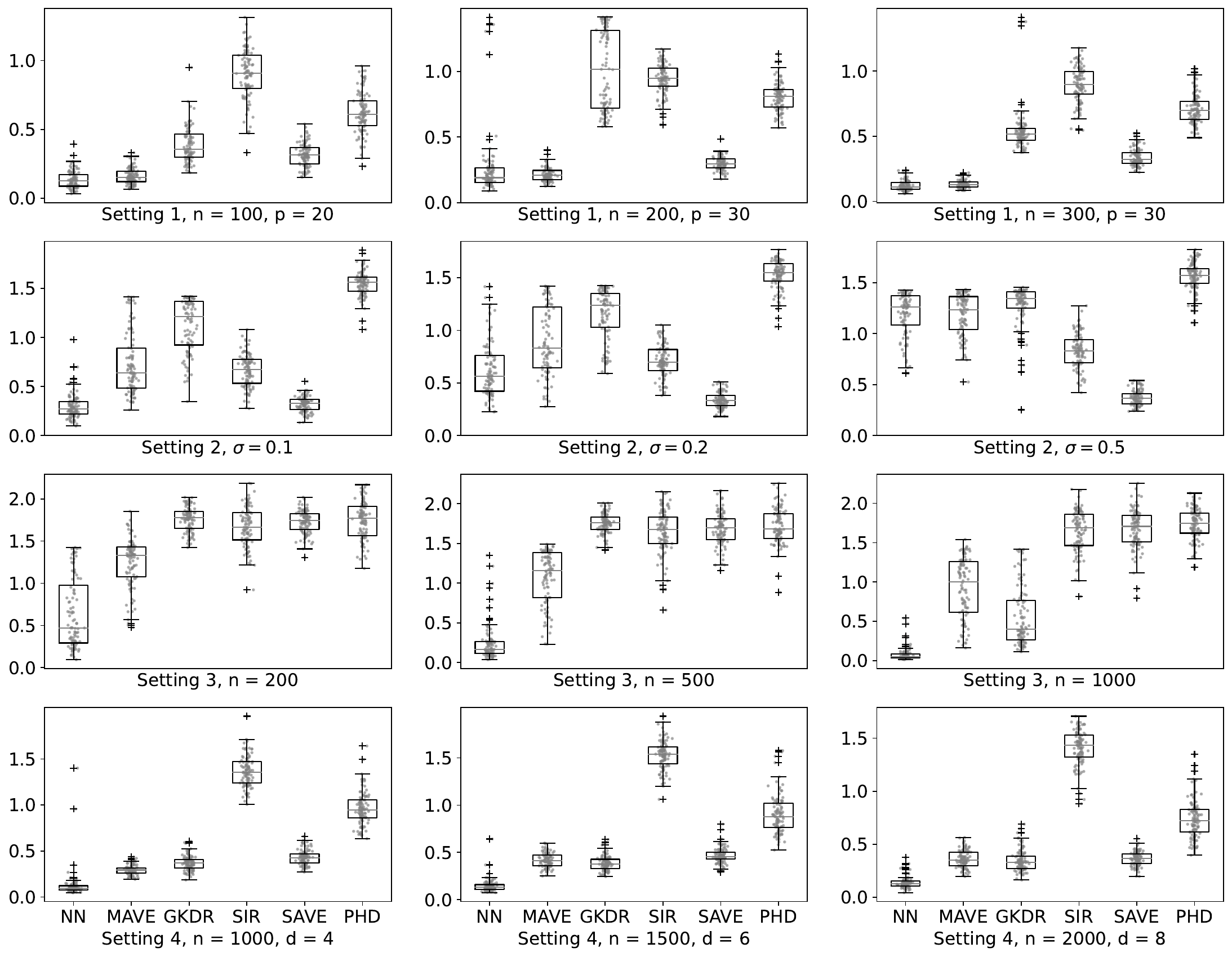}
\caption{Boxplots of $\|\pi_{\hat{B}}-\pi_{B_0}\|_F$ on 100 replicates across different methods. In each panel, six methods for SDR are included, among which NN represents the neural network-based method.}
\label{Fig: simulation}
\end{figure*}

\section{Real Data Analysis}

We further applied the neural network-based method involving rank regularization to a real regression task. In practice, the precise intrinsic dimension $d$ is unknown, and it is uncertain whether a low-dimensional structure exists. To address this issue, we used cross validation to determine an appropriate $d$ from the range of values $\{1, \ldots, p\}$. More specifically, the raw data set was divided into 80\% training data and 20\% testing data. The optimal value of $d$ was determined through cross validation on the training data. Subsequently, the final model was fit using the selected $d$ on the training data, and the mean squared error on the testing data was evaluated.

In order to reduce randomness, the aforementioned process was repeated 20 times and the resulting testing errors were recorded. Additionally, we conducted a comparative analysis between the neural network-based approach and alternative methods including vanilla neural network without rank regularization, SIR-based regression, SAVE-based regression, and MAVE-based regression. For the latter three techniques, we initially executed SDR to acquire the embedded data, followed by the utilization of a fully-connected neural network for predictive purpose. The optimal value for $d$ was also determined through cross validation.

We utilized a data set of weather records from Seoul, South Korea during the summer months from 2013 to 2017 \citep{cho2020comparative}, available at the UCI data set repository (bias correction of numerical prediction model temperature forecast).
This data set contained 7750 observations with 23 predictors and 2 responses, specifically the next-day maximum air temperature and next-day minimum air temperature. After excluding the variables for station and date, the data set was reduced to 21 predictors, which were further standardized using StandardScaler in scikit-learn package.

\begin{table}[t]
\begin{tabular}{lrrrrr}
\toprule
    & NN-RR & NN-VN & MAVE & SIR & SAVE  \\ \midrule
mean & 0.602 & 0.774 & 0.743 & 1.324 & 0.772  \\ 
std & 0.043 & 0.116 & 0.159 & 0.190 & 0.161  \\ 
$\bar{d}$ & 19.6 & --- & 19.25 & 19.6 & 20.6 \\ 
\bottomrule
\end{tabular}
\caption{The results of testing errors on 20 replicates across different methods. We report the average and standard deviation of testing errors, along with averaged optimal $d$ determined through cross validation. NN-RR, NN-VN, MAVE, SIR, SAVE represent the neural network-based method with rank regularization, vanilla neural network regression, MAVE-based regression, SIR-based regression, and SAVE-based regression, respectively.}
\label{Tab: real_data}
\end{table}

It is evident from Table \ref{Tab: real_data} that the neural network-based method with rank regularization outperformed other methods, demonstrating the effectiveness of the modification compared to the vanilla neural network, and the sound performance in reducing dimensions compared to other SDR methods. The presence of latent structures was partially supported by the averaged optimal $d$. It was possible that 19 or 20 combinations of raw predictors might be sufficient, as opposed to the original 21 predictors.

\section{Proofs}


\subsection{Proof of Theorem 1}

Define the vanilla neural network function class without rank regularization as
$$
\begin{aligned}
& \mathcal{F}_{\mathcal{L}, \mathcal{M}, \mathcal{S}, \mathcal{R}}^*=\Big\{ \\
& f(x)=\phi_L\circ\sigma_{L-1}\circ\phi_{L-1}\circ\cdots \circ\sigma_0\circ\phi_0(x):  \\
& \phi_i(z)=W_i^{\T}z+b_i, W_i\in\real^{d_i\times d_{i+1}}, i=0, \ldots, L, \\
& \text{with } L=\mathcal{L}, \max_{i=1, \dots, L}d_i=\mathcal{M}, \sum_{i=0}^L(d_i+1)d_{i+1}=\mathcal{S}, \\
& \|f\|_{L^{\infty}([0, 1]^d)}\le \mathcal{R}
\Big\}.
\end{aligned}
$$

\begin{lemma}
Suppose that $h$ is H\"older continuous with $\alpha\in (0, 1]$ and $\lambda>0$. Then, for any $\zeta>0$, there exists a function $g$ in neural network function class $\mathcal{F}_{\mathcal{L}, \mathcal{M}, \mathcal{S}, \mathcal{R}}^*$, with rectified linear unit activation and $\mathcal{L}, \mathcal{M}, \mathcal{S}, \mathcal{R}$ large enough, such that
$$
\|h-g\|_{L^{\infty}([0, 1]^d)}<\zeta.
$$
\label{Lem: nn_approx}
\end{lemma}

We note that Lemma \ref{Lem: nn_approx} is a simplified version of Theorem 1.1 in \cite{shen2019deep}.

\begin{proof}[Proof of Theorem 1]
For $f\in\mathcal{F}_{\mathcal{L}, \mathcal{M}, \mathcal{S}, \mathcal{R}}$ such that $\pi_{\mathcal{T}(f)}\ne \pi_{B_0}$, under Assumption \ref{Asm: f0_sharpness}, there exists a scalar $t>0$ satisfying
$$
\E\left\{\var\left[f_0(B_0^{\T} x)| \mathcal{T}(f)^{\T} x\right]\right\}>t.
$$
Then, we have
$$
\begin{aligned}
\E[y-f(x)]^2 =& \E[\epsilon+f_0(B_0^{\T} x)-f(x)]^2 \\
=& \E(\epsilon^2)+\E[f_0(B_0^{\T} x)-f(x)]^2 \\
=& \nu^2+\var[f_0(B_0^{\T} x)-f(x)]\\
& +\E^2[f_0(B_0^{\T} x)-f(x)] \\
\ge & \nu^2+\var[f_0(B_0^{\T} x)-f(x)] \\
\ge & \nu^2+\E\{\var[f_0(B_0^{\T} x)| \mathcal{T}(f)^{\T} x]\}\\
>& \nu^2+t.
\end{aligned}
$$
On the other hand, for $\tilde{f}\in\mathcal{F}_{\mathcal{L}, \mathcal{M}, \mathcal{S}, \mathcal{R}}$ such that $\pi_{\mathcal{T}(\tilde{f})}=\pi_{B_0}$, there exists an orthogonal matrix $Q\in\real^{d\times d}$ such that $B_0=\mathcal{T}(\tilde{f})Q$. Hence, $\tilde{f}(x)=\tilde{f}\circ \mathcal{T}(\tilde{f})\circ Q(B_0^{\T} x)$ and $\tilde{f}\circ \mathcal{T}(\tilde{f})\circ Q$ is still a neural network function. Assumption \ref{Asm: f0_smoothness} and Lemma \ref{Lem: nn_approx} imply that there is a neural network function $g$ satisfying
$$
\|f_0(B_0^{\T} x)-g(B_0^{\T} x)\|_{L^{\infty}(\mathcal{B}([0, 1]^p))}<t^{1/2}/2,
$$
for sufficiently large $\mathcal{L}, \mathcal{M}, \mathcal{S}, \mathcal{R}$.
As a result, select $\tilde{f}$ such that $\tilde{f}\circ \mathcal{T}(\tilde{f})\circ Q=g$, and it follows that
$$
\E[y-\tilde{f}(x)]^2=\nu^2+\E[f_0(B_0^{\T} x)-g(B_0^{\T} x)]^2<\nu^2+t/4\ .
$$
Therefore, for any $f\in\mathcal{F}_{\mathcal{L}, \mathcal{M}, \mathcal{S}, \mathcal{R}}$ such that $\pi_{\mathcal{T}(f)}\ne \pi_{B_0}$, there exists a neural network function $\tilde{f}$, with the same rank regularization, satisfying 
$$
\E[y-\tilde{f}(x)]^2< \E[y-f(x)]^2.
$$
To conclude, $\pi_{\mathcal{T}(f^*)}=\pi_{B_0}$, which entails that $\Pi_{\mathcal{T}(f^*)}=\Pi_{B_0}$.
\end{proof}

\subsection{Proof of Theorem 2}

Let $L(f)=\E[y-f(x)]^2$. Without loss of generality, suppose $\mathcal{R}=\mathcal{R}_0$. We first present some useful lemmas.

\begin{lemma}
For sufficiently large $n$, it follows that
$$
\sup_{f\in\mathcal{F}_{\mathcal{L}, \mathcal{M}, \mathcal{S}, \mathcal{R}}}|L(f)-L_n(f)|=O_p\left(\sqrt{\frac{\mathcal{S}\mathcal{L}\log(\mathcal{S})(\log n)^5}{n}}\right).
$$
\label{Lem: function_complexity}
\end{lemma}

The proof of Lemma \ref{Lem: function_complexity} is presented in the Supplementary Material. We note this convergence rate may not be optimal, but is sufficient for deducing consistency. 

\begin{lemma}
Under Assumptions \ref{Asm: f0_smoothness} and \ref{Asm: y_subexp}, for sufficiently large $n$, we have
$$
\begin{aligned}
\E[y-\hat{f}_n(x)]^2\le \nu^2+Cn^{-2\alpha/(d+4\alpha)}(\log n)^3,
\end{aligned}
$$
where $C$ is a constant depending on $\mathcal{R}_0$ and $d$.
\label{Lem: excess_risk}
\end{lemma}

\begin{proof}[Proof of Lemma \ref{Lem: excess_risk}]
We begin with the decomposition that
$$
\begin{aligned}
L(\hat{f}_n) \le & \mathbb{E}\left[L_n(\hat{f}_n)+\sup_{f\in\mathcal{F}_{\mathcal{L}, \mathcal{M}, \mathcal{S}, \mathcal{R}}}|L(f)-L_n(f)|\right] \\
\le & \mathbb{E}\Bigg[\inf_{f\in\mathcal{F}_{\mathcal{L}, \mathcal{M}, \mathcal{S}, \mathcal{R}}}L_n(f\circ \mathcal{T}(f)\circ B_0^{\T}) \\
& +\sup_{f\in\mathcal{F}_{\mathcal{L}, \mathcal{M}, \mathcal{S}, \mathcal{R}}}|L(f)-L_n(f)|\Bigg] \\
\le & \mathbb{E}\Bigg[\inf_{f\in\mathcal{F}_{\mathcal{L}, \mathcal{M}, \mathcal{S}, \mathcal{R}}}L(f\circ \mathcal{T}(f)\circ B_0^{\T})\\
& +2\sup_{f\in\mathcal{F}_{\mathcal{L}, \mathcal{M}, \mathcal{S}, \mathcal{R}}}|L(f)-L_n(f)|\Bigg] \\
\le &\inf_{f\in\mathcal{F}_{\mathcal{L}, \mathcal{M}, \mathcal{S}, \mathcal{R}}}L(f\circ \mathcal{T}(f)\circ B_0^{\T}) \\
& +C_1\sqrt{\mathcal{S}\mathcal{L}\log(\mathcal{S})(\log n)^5n^{-1}},
\end{aligned}
$$
where $C_1$ is a constant, for sufficiently large $n$. Hence, 
$$
\begin{aligned}
& L(\hat{f}_n)-\nu^2 \\
=& L(\hat{f}_n)-L(f_0\circ B_0^{\T}) \\
=& L(\hat{f}_n)-\inf_{f\in\mathcal{F}_{\mathcal{L}, \mathcal{M}, \mathcal{S}, \mathcal{R}}}L(f\circ \mathcal{T}(f)\circ B_0^{\T}) \\
& +\inf_{f\in\mathcal{F}_{\mathcal{L}, \mathcal{M}, \mathcal{S}, \mathcal{R}}}L(f\circ \mathcal{T}(f)\circ B_0^{\T})-L(f_0\circ B_0^{\T}) \\
\le& C_1\sqrt{\mathcal{S}\mathcal{L}\log(\mathcal{S})(\log n)^5n^{-1}} \\
& +\inf_{f\in\mathcal{F}_{\mathcal{L}, \mathcal{M}, \mathcal{S}, \mathcal{R}}}\E[f_0(B_0^{\T} x)-f\circ\mathcal{T}(f)(B_0^{\T} x)]^2 \\
\le& C_1\sqrt{\mathcal{S}\mathcal{L}\log(\mathcal{S})(\log n)^5n^{-1}} \\
& +\inf_{g\in\mathcal{F}_{\mathcal{L}, \mathcal{M}, \mathcal{S}, \mathcal{R}}^*}\|f_0-g\|_{L^{\infty}([0, 1]^d)}^2.
\end{aligned}
$$
By Theorem 1.1 in \citet{shen2019deep}, there exists $g^*$ with $\mathcal{M}=3^{d+3}\max(d\lfloor K^{1/d}\rfloor, K+1), \mathcal{L}=12D+14+2d$ for some constant $K, D>0$ such that
$$
\|f_0-g^*\|_{L^{\infty}([0, 1]^d)}\le 19\sqrt{d}\lambda(KD)^{-2\alpha/d}\ .
$$
Let 
$$
\begin{aligned}
D=O\left(n^{d/(2d+8\alpha)}\right), \quad K=O(1),
\end{aligned}
$$
and observe that $\mathcal{S}=O(\mathcal{M}^2\mathcal{L})$. Then, it follows that
$$
\begin{aligned}
L(\hat{f}_n) &= \E[y-\hat{f}_n(x)]^2 \\
&\le \nu^2+Cn^{-2\alpha/(d+4\alpha)}(\log n)^3,
\end{aligned}
$$
where $C$ is a constant depending on $\mathcal{R}_0$ and $d$.
\end{proof}

With Lemma \ref{Lem: excess_risk}, the proof of Theorem 2 is a direct application of Theorem 5.9 in \citet{van2000asymptotic}.

\begin{proof}[Proof of Theorem 2]
Recall that $m_0=f_0\circ B_0^{\T}$. Let $R(f)=L(f)-L(m_0)=\E[f_0(B_0^{\T} x)-f(x)]^2\ge 0$ and $R_n(f)=L_n(f)-L_n(m_0)$. Then, Lemma \ref{Lem: function_complexity} indicates that
$$
\sup_{f\in\mathcal{F}_{\mathcal{L}, \mathcal{M}, \mathcal{S}, \mathcal{R}}}|R_n(f)-R(f)|\to 0, \quad \text{in probability},
$$
yielding $R_n(\hat{f}_n)=o_p(1)$ by incorporating with Lemma \ref{Lem: excess_risk}.

We denote the metric 
$$
d_1(f, m_0)=\min_{Q\in\mathcal{Q}}\|B_0-\mathcal{T}(f)Q\|_2\vee\|f_0\circ Q-f\circ \mathcal{T}(f)\|_{L^2(\mu)}.
$$
Here, $f\in\mathcal{F}_{\mathcal{L}, \mathcal{M}, \mathcal{S}, \mathcal{R}}$, $m_0=f_0\circ B_0^{\T}$, $a\vee b$ means $\max(a, b)$, and $\mu$ is the probability distribution of $\mathcal{T}(f)^{\T} x$. Recall that $d(B, B_0)=\min_{Q\in \mathcal{Q}}\|B_0-B Q\|_2$, where $B\in\Psi_d$. 

For $f$ and $\delta>0$ such that $d_1(f, m_0)\ge \delta$, if 
$$
d(\mathcal{T}(f), B_0)\le\tilde{\delta}=\min\{[\delta^2/(8\mathcal{R}_0\lambda)]^{1/\alpha}, \delta/2\},
$$
we have 
$$
\begin{aligned}
|f_0(B_0^{\T} x)-f_0(Q^{\T}\mathcal{T}(f)^{\T} x)| &\le \lambda\|B_0^{\T} x-Q^{\T}\mathcal{T}(f)^{\T} x\|_2^{\alpha} \\
&\le \lambda\|B_0^{\T} -Q^{\T}\mathcal{T}(f)^{\T} \|_2^{\alpha} \\
&\le \frac{\delta^2}{8\mathcal{R}_0},
\end{aligned}
$$
for some orthogonal matrix $Q$. Hence, 
$$
\begin{aligned}
R(f) &= \E[f_0(B_0^{\T} x)-f(x)]^2 \\
&\ge \E[f_0(Q^{\T}\mathcal{T}(f)^{\T} x)-f(x)]^2\\
& \hspace{2em} +2\E[f_0(B_0^{\T} x)-F_0][F_0-f(x)] \\
&\ge -\frac{\delta^2}{2}+d_1(f, m_0)^2\ge \frac{\delta^2}{2}>0,
\end{aligned}
$$
where $F_0=f_0(Q^{\T}\mathcal{T}(f)^{\T} x)$.

For the case that $d(\mathcal{T}(f), B_0)>\tilde{\delta}$, by Assumption \ref{Asm: metric_sharpness}, we have $R(f)\ge \xi>0$ for some positive constant $\xi$. To conclude, we obtain 
$$
\inf_{f: d_1(f, m_0)\ge \delta}R(f)>0,
$$
for any $\delta>0$. Finally, applying Theorem 5.9 in \citet{van2000asymptotic}, $R_n(\hat{f}_n)=o_p(1)$ implies that $d_1(\hat{f}_n, m_0)\to 0$ in probability. Hence, $d(\mathcal{T}(\hat{f}_n), B_0)\to 0$ in probability.
\end{proof}

\section{Discussion}

We have demonstrated that neural networks attain the capability of detecting the underlying low-dimensional structure that preserves all information in $x$ about the conditional mean function $\E(y|x)$. As a result, neural networks are suitable to be utilized for the estimation of central mean subspace $\Pi_{B_0}$. The theoretical investigations sharpen our understanding of neural networks, while broadening the scope of SDR as well.

In the context of SDR, a more general scenario than sufficient mean dimension reduction emerges when considering
\begin{equation} 
y\indep x | B_0^{\T}x.
\label{Eq: CS} 
\end{equation}
And the column space spanned by $B_0$ corresponds to the central subspace \citep{cook1998regression, li2018sufficient}.
It is clear that relation (\ref{Eq: CS}) is equivalent to that $p(y| x)=p(y| B_0^{\T}x)$, where $p(\cdot| \cdot)$ represents the conditional probability density function. Following the work of \citet{xia2007constructive}, we can further adapt neural networks to estimate the central subspace by modifying the loss function.

Under mild conditions, \citet{xia2007constructive} showed that
$$
\E[K_h(y-y_0)| x=x_0]\to p(y_0 | B_0^{\T}x_0), \quad \text{as }h\to 0^+.
$$
Here, $(x_0,y_0)$ is a fixed point, and $K_h(\cdot)$ is a suitable kernel function with a bandwidth $h$. {Such finding then implies that 
$$
\begin{aligned}
K_h(y-y_0) &= \E[K_h(y-y_0)| x]+\text{remainder term} \\
& \approx p(y_0 | B_0^{\T}x)+\text{remainder term}.
\end{aligned}
$$
Based on this discovery, it is natural to employ a neural network function $f(B^{\T}x, y_0): \mathbb{R}^{d+1}\to\mathbb{R}$ to approximate $p(y_0 | B_0^{\T}x)$, and obtain the estimate of $B_0$ at the population level by solving the following problem 
$$
(B^*, f^*)=\argmin{B, f}\E[K_h(y-\tilde{y})-f(B^{\T}x, \tilde{y})]^2,
$$
where $\tilde{y}$ is an independent copy of $y$.
} Empirically, we define the loss function as 
$$
\tilde{L}_n(B, f)=\frac{1}{n^2}\sum_{i=1}^n\sum_{j=1}^n[K_h(y_j-y_i)-f(B^{\T}x_j, y_i)]^2.
$$

We provide some additional simulation results to verify the feasibility utilizing neural networks for the estimation of central subspace; see the Supplementary Material. Theoretical analysis of neural networks for the estimation of the central subspace, including unbiasedness and consistency, deserves further studies.

\bibliography{aaai25}

\end{document}